\documentclass[letterpaper, 10 pt, conference]{ieeeconf}  

\IEEEoverridecommandlockouts         

\usepackage[doi=false,isbn=false,defernumbers=true,style=ieee,backend=bibtex,maxbibnames=1000]{biblatex} 
\bibliography{library_GXZ.bib}

\usepackage{balance} 

\usepackage{xcolor,calc}

\usepackage[pdftex]{graphicx}
\usepackage{amsmath} 
\usepackage{amssymb}  
\usepackage{color}
\usepackage{enumerate}
\usepackage{gensymb}
\usepackage{multirow}
\usepackage{mathtools}
\usepackage{booktabs}
\usepackage{epstopdf}
\usepackage{algorithm,algcompatible}

\usepackage[normalem]{ulem}

\usepackage{url}
\usepackage{bm}

\usepackage{booktabs}
\newcommand{\ra}[1]{\renewcommand{\arraystretch}{#1}}

\newtheorem{theorem}{Theorem}

\newtheorem{proposition}{Proposition}

\makeatletter
\newsavebox\myboxA
\newsavebox\myboxB
\newlength\mylenA

\newcommand*\xoverline[2][0.75]{%
    \sbox{\myboxA}{$\m@th#2$}%
    \setbox\myboxB\null
    \ht\myboxB=\ht\myboxA%
    \dp\myboxB=\dp\myboxA%
    \wd\myboxB=#1\wd\myboxA
    \sbox\myboxB{$\m@th\overline{\copy\myboxB}$}
    \setlength\mylenA{\the\wd\myboxA}
    \addtolength\mylenA{-\the\wd\myboxB}%
    \ifdim\wd\myboxB<\wd\myboxA%
       \rlap{\hskip 0.5\mylenA\usebox\myboxB}{\usebox\myboxA}%
    \else
        \hskip -0.5\mylenA\rlap{\usebox\myboxA}{\hskip 0.5\mylenA\usebox\myboxB}%
    \fi}
\makeatother

\newcommand{\R}{\mathbb{R}}

\newcommand*{\blue}[1]{\textcolor{black}{#1}}

\usepackage{tabularx}
\usepackage{booktabs}

\makeatletter
\let\NAT@parse\undefined
\makeatother
\usepackage{hyperref}

\title{\LARGE \bf Safe and Near-Optimal Policy Learning for Model \\ Predictive Control  using Primal-Dual Neural Networks}

\author{Xiaojing Zhang, Monimoy Bujarbaruah, and Francesco Borrelli  
\thanks{The authors are with the MPC Lab, UC Berkeley, USA; E-mails: \tt\scriptsize{\{xiaojing.zhang, monimoyb, fborrelli\}@berkeley.edu.}} %
}

\begin{document}

\maketitle
  \thispagestyle{empty}
\pagestyle{empty}

\begin{abstract}
In this paper, we propose a novel framework for approximating the explicit MPC law for linear parameter-varying systems using supervised learning. In contrast to most existing approaches, we not only learn the control policy, but also a ``certificate policy", that allows us to estimate the sub-optimality of the learned control policy online, during execution-time. We learn both these policies from data using supervised learning techniques, and also provide a randomized method that allows us to guarantee the quality of each learned policy, measured in terms of feasibility and optimality. This in turn allows us to bound the probability of the learned control policy of being infeasible or suboptimal, where the check is performed by the certificate policy.
Since our algorithm does not require the solution of an optimization problem during run-time, it can be deployed even on resource-constrained systems. We illustrate the efficacy of the proposed framework on a vehicle dynamics control problem where we demonstrate a speedup of up to two orders of magnitude compared to online optimization with minimal performance degradation.
\end{abstract}

\section{Introduction}
Model Predictive Control (MPC) is an advanced control strategy that is able to optimize a plant's behavior while respecting system constraints. Originating in process control, MPC has found application in fields such as building climate control \cite{Oldewurtel:energyBuildings:2012, zhang:schildbach:sturzenegger:morari:13, MaMatuskoBorrelli2014}, quadcopter control \cite{BouffardAswani2012}, self-driving vehicles \cite{Liniger_OCAM2015, CarrauLinigerZhangLygerosECC2016, rosolia2017autonomousrace, bujarbaruahtorque, ZhangLinigerBorrelli_CollAvoid_2017, 2018arXiv180604335B} and robotics in general \cite{LevineAbbeel2016learning}. However, implementing MPC  on fast dynamical systems with limited computation capacity is generally challenging since MPC requires the solution of an optimization problem at each sampling step. This is especially true for mass-produced systems such as drones and automotive vehicles. 

Over the past decades, significant research effort has been devoted to enabling MPC to systems with limited computation power by developing numerically efficient solvers that exploit the structure of the MPC optimization problem \cite{MattingleyBoydCVXGen2012, osqp}. Another approach to reduce computation load of MPC is to pre-compute the optimal control law offline, store it into the system, and evaluate it during run-time. This approach, known as \emph{Explicit MPC}, is generally well-understood for linear time-invariant system where the optimal control law has been shown to be piecewise affine over polyhedral regions \cite{BorrelliBemporadMorari_book}. The main drawbacks of Explicit MPC is that synthesis of the optimal control law can be computationally demanding even for medium-sized problems, and storing and evaluating the look up tables can be prohibitive for embedded platforms \cite{KvasnicaFikarTAC2012}. To address this issue, significant effort has been devoted to computing \emph{suboptimal} explicit MPC polices that are defined over fewer polyhedral regions and hence can be evaluated more efficiently \cite{JohansenGrancharovaTAC2003, multiresSeanSummers, JonesMorariPolytopicApproximationTAC2010, KvasnicaFikarClippingTAC2012, KvasnicaAutomatica2013}.
Another way of approximating explicit MPC control laws is by means of \emph{function approximation}  such as supervised learning \cite{ParisiniNNAutomatica1995} or reinforcement learning \cite{ChenMorariACC2018}. In supervised learning, for example, the goal is to find a function, out of a given function class, that best explains some given training data. One main advantage of using such function approximation is that, while training can be computationally demanding, evaluating the control law can often be carried out very efficiently \cite{BemporadUltraFastTAC2011,DomahidiLearning2011}. Although techniques such as supervised learning can in principle apply to nonlinear systems, most research  has focused on linear time-invariant systems \cite{BemporadUltraFastTAC2011, DomahidiLearning2011, ChenMorariACC2018, LuciaEfficientRepresentationArXiv2018}, since guaranteeing safety and performance of the approximated control law is generally hard for nonlinear systems \cite{HertneckAllgowerArXiv2018}.

In this paper, we propose a novel policy approximation scheme for learning the explicit MPC control law using a \emph{primal policy} and a \emph{dual policy}.
The policies are trained and verified offline using randomly generated training and verification data, respectively. Online, in real-time, the dual policy is then used to estimate the performance of the control action given by the primal policy.
\blue{This is in contrast to most existing methods that incorporate safety constraints during the policy learning phase, which may result in  suboptimal controllers.} 
Specifically, our contributions can be summarized as follows:

\begin{itemize}
    \item \textcolor{black}{For the offline phase, we propose a supervised learning scheme to train the primal and dual policies, and introduce a randomized verification  methodology to estimate the quality (i.e., feasibility and suboptimality) of the trained policies.
    Given an admissible probability of quality violation, explicit sample sizes are provided for the verification step.}

    \item We show how, during the online phase, we can use the dual policy to track the quality (i.e., feasibility and suboptimality) of the approximated MPC law \blue{i.e. the primal policy}, using ideas from duality theory of convex optimization. 
    If the check fails, then a backup controller is used.
    


    \item We demonstrate the efficacy of the proposed primal-dual policy learning framework for an integrated chassis control problem in vehicle dynamics. In particular, we demonstrate computation speedups of up to 100x, when compared to state-of-the-art numerical solvers, potentially enabling the implementation of MPC on mass-produced embedded systems.
\end{itemize}
In contrast to most existing methods, our methodology also applies to linear parameter-varying systems. We also stress that the proposed framework is applicable to any function approximation scheme. In that regard, our methodology is able to check safety and performance of control laws that are encoded through deep neural networks (DNN). 


\section{Problem Description}\label{sec:problemDescription}
\subsection{Dynamics, constraints, and control objective}
We consider linear parameter-varying (LPV) systems of the form
\begin{equation}\label{eq:paramVaryingSystem}
    x_{k+1} = A(q_k)x_k + B(q_k)u_k,
\end{equation}
where $x_k\in\R^{n_x}$ is the state at time $k$, $u_k\in\R^{n_u}$ is the input at time $k$, and $A(q_k)$ and $B(q_k)$ are known matrices of appropriate dimensions, that depend on a time-varying parameter $q_k$. Throughout this paper, we assume that the parameter $q_k$ is known. The system is subject to linear input and state constraints of the form
\begin{equation}\label{eq:constraints}
    \mathbb U := \{u: H_u u_k \leq h_u\}, \qquad \mathbb X := \{x: H_x x_k \leq h_x\},
\end{equation}
for given $H_u$, $h_u$, $H_x$ and $h_x$.
At each time step $k$, the control objective is to minimize, over a finite horizon $T$, a quadratic cost of the form
\begin{equation}\label{eq:objFunction}
    x_T^\top Q_f x_T + \sum_{k=0}^{T-1} x_k^\top Q x_k + u_k^\top R u_k,
\end{equation}
where the matrices $Q$ and $Q_f$ are assumed to be positive semi-definite and $R$ is chosen to be positive definite.

\subsection{Model Predictive Control}
At each time step $t$, Model Predictive Control (MPC) measures the state $x_t$ and solves the following finite-horizon optimal control problem
\begin{equation} \label{eq:MPCproblemOrig}
    \begin{array}{llll}
        \displaystyle \min_{U_t, X_t} & x_{N|t}^\top Q_f x_{T|t} + \displaystyle \sum_{k=0}^{T-1} x_{k|t}^\top Q x_{k|t} + u_{k|t}^\top R u_{k|t}  \\
        \ \ \text{s.t.} & x_{k+1|t} = A(q_{k|t}) x_{k|t} + B(q_{k|t}) u_{k|t}, \\                    & (x_{k|t}, u_{k|t} ) \in \mathbb X \times \mathbb U,\  x_{T|t} \in \mathbb X_f, \\
                    & x_{0|t} = x_t,\ \ k=0,\ldots,T-1,
    \end{array}
\end{equation}
where $x_{k|t}$ is the state at time $t+k$ obtained by applying the predicted inputs $u_{0|t},\ldots,u_{k-1|t}$ to system~\eqref{eq:paramVaryingSystem}. Furthermore, $U_t := [u_{0|t} ,\ldots, u_{T-1|t}]$ and $X_t := [x_{0|t} ,\ldots, x_{T|t}]$ are the collection of all predicted inputs and states, respectively. The set $\mathbb X_f \subset\R^{n_x}$, which we assume is a compact polytope, is a so-called terminal set, and ensures recursive feasibility of the MPC controller, see e.g.\ \cite{BorrelliBemporadMorari_book} for details. If $U^*_t$ is a minimizer of \eqref{eq:MPCproblemOrig}, then MPC applies the first input $u_t = u^*_{0|t}$. This process is repeated at the next time step, resulting in a receding horizon control scheme. 

By eliminating the states $X_t$ from \eqref{eq:MPCproblemOrig}, we can express \eqref{eq:MPCproblemOrig} compactly as
\begin{equation}\label{eq:MPC_primal}
    \begin{array}{rlll}
        J^*(P_t) := \displaystyle \min_{U} & \frac{1}{2} U^\top Q(P_t) U + c(P_t)^\top U  \\
        \ \ \text{s.t.} & H(P_t) U \leq h(P_t),
    \end{array}
\end{equation}
where $P_t := [x_t,q_{0|t},\ldots,q_{T-1|t}]$ is the collection of all parameters $\{q_{k|t}\}_k$ and the initial state, and $Q(P_t)$, $c(P_t)$, $H(P_t)$, $h(P_t)$ are appropriately defined matrices, see e.g.,~\cite{Goulart2006} for their construction. 
\blue{It is assumed that, at each time step $t$, the parameters $P_t$ are known. In practice, they may come from an external estimator.}
We point out that \eqref{eq:MPC_primal} is a \emph{multi-parametric quadratic program}, whose optimizer $U^*(\cdot)$ and optimal value $J^*(\cdot)$ depend on $P_t$ \cite{BorrelliBemporadMorari_book}. To streamline the upcoming presentation, we assume in this paper that the parameters $P_t$ take values in a compact set $\mathcal{P}$, and that \eqref{eq:MPC_primal} is feasible and finite for all $P_t \in \mathcal{P}$.

Solving the optimization problem~\eqref{eq:MPC_primal} at each sampling time can be computationally challenging for fast systems on resource constrained platforms. To address this issue, we propose the use of function approximation to \blue{offline} \emph{learn} an approximate policy $\tilde U_\theta(\cdot) \approx U^*(\cdot)$ (``primal policy"), as well as an run-time optimality certificate via a so-called ``dual policy".

\section{Technical Background}\label{sec:background}

\subsection{Supervised Learning}\label{sec:supLearning}
The goal in classical function approximation is to approximate a function $f\in\mathcal{F}$, defined on some function space $\mathcal{F}$, by another function $\tilde f\in\tilde{\mathcal{F}}\subset \mathcal F$ such that $\|f - \tilde f\|_\mathcal{F}$ is minimized. Since this minimization problem is performed over the infinite dimensional space of functions $\tilde{\mathcal F}$, it is generally intractable. A common approach is to restrict oneself to function spaces $\tilde{\mathcal F}=\tilde{\mathcal F}_\theta$ that are defined by a finite number of parameters $\theta$, and to approximate the norm $\|\cdot\|_{\mathcal F}$ by the empirical error. This is achieved by drawing $M$ samples $\{z^{(i)},f(z^{(i)})\}_{i=1}^M$, upon which the finite-dimensional problem of \emph{learning $f(\cdot)$} is given by
\begin{equation}\label{eq:Function_Approx}
    \theta^* := \displaystyle \arg\min_\theta\ \ \sum_{i=1}^M \mathcal{L}\left( f(z^{(i)}) - \tilde{f}_\theta(z^{(i)}) \right), 
\end{equation}
where $\theta^*$ denotes the optimal parameter and $\mathcal{L}(\cdot)$ a loss function, such as the euclidean norm. The choice of the loss function and the function space $\tilde{\mathcal{F}}_\theta$ is often problem-dependent. \blue{Typical function classes} \blue{include Deep Neural Networks and weighted sums of basis functions, see \cite{Smola2004, Genton:2002} for examples. }



\subsection{Duality Theory}\label{sec:dualitySection}
Duality is used in optimization to certify optimality of a given (candidate) solution. Specifically, to every convex optimization problem $p^* := \min_x\{f(x): h(x)\leq 0\}$, we can associate its \emph{dual problem} $d^* := \max_{\lambda}\{g(\lambda): \lambda\geq0\}$, where $g(\lambda) := \min_x\{f(x) - \lambda^\top h(x)\}$. Under appropriate technical assumptions\footnote{These include feasibility, finite optimal value, and constraint qualifications, see \cite[Chapter~5]{boyd2004convex} for details.}, it can be shown that $p^* = d^*$ (``strong duality"). In this paper, we will make use of the \emph{weak duality} property that, for every primal feasible point $x$ and every dual feasible $\lambda$, it holds
\begin{equation}\label{eq:weakDuality}
    g(\lambda) \leq f(x).
\end{equation}
Notice that \eqref{eq:weakDuality} can be used to bound the suboptimality of a candidate solution $\bar x$, since $f(\bar x) - p^* \leq f(\bar x) - g(\lambda)$, for any $\lambda\geq0$. 

\subsubsection*{Dual of \eqref{eq:MPC_primal}}
It is well-known that the dual of a convex quadratic optimization problem is again a convex quadratic optimization problem \cite{boyd2004convex}. Specifically, the dual of \eqref{eq:MPC_primal}, is given by
\begin{equation}\label{eq:MPC_dual}
    \begin{array}{rlll}
        D^*(P_t) := \displaystyle \min_{\lambda_t} & \frac{1}{2} \lambda_t^\top \tilde Q(P_t) \lambda_t + \tilde c(P_t)^\top \lambda_t + \tilde g(P_t)  \\
        \text{s.t.}\ & \lambda_t \geq 0,
    \end{array}
\end{equation}
where $\tilde Q(\cdot)$, $\tilde c(\cdot)$ and $\tilde{g}(\cdot)$ depend on $P_t$. Notice that, similar to \eqref{eq:MPC_primal}, the optimizer of \eqref{eq:MPC_dual} depends on the parameters $P_t$, i.e, $\lambda_t^* = \lambda_t^*(P_t)$. Furthermore, since \eqref{eq:MPC_primal} is convex, it follows from strong duality that $J^*(P_t) = D^*(P_t)$.

\section{Primal-Dual Policy Learning}\label{sec:primalDualPolicyLearning}
In this section, we present our primal-dual policy learning framework, where we learn both a primal policy $\tilde{U}_{\theta_p}(\cdot) \approx U^*(\cdot)$ and a dual policy $\tilde\lambda_{\theta_d}(\cdot) \approx \lambda^*(\cdot)$. We show how these approximated functions can be used to efficiently ensure feasibility and near-optimality of the control law during run-time of the controller. 

\subsection{Primal and Dual Learning Problems}
We use supervised learning to approximate  the primal policy $U^*(\cdot)$ and dual policy $\lambda^*(\cdot)$. To this end, we generate $M$ samples $\{P^{(i)}  , U^*(P^{(i)}), \lambda^*(P^{(i)})\}_{i=1}^M$, where $P^{(i)}\in\mathcal{P}$ are extracted according to some \blue{user-chosen} distribution $\mathbb P$, and $U^*(P^{(i)})$ and $\lambda^*(P^{(i)})$  are obtained by solving \eqref{eq:MPC_primal} and \eqref{eq:MPC_dual}, respectively. \blue{The choice of the distribution $\mathbb P$ in general will depend on the task.  One could, for example, bias the distribution around a nominal operating point. If no such operating point is known, then the uniform distribution over $\mathcal{P}$ can be chosen.}

Given the samples, the \emph{primal learning problem} is
\begin{subequations}
\begin{equation}\label{eq:primalLearning}
        \theta_p^* := \arg\min_{\theta_p} \ \displaystyle\sum_{i=1}^M \mathcal{L}\left(  \tilde U_{\theta_p}(P^{(i)}) - U^*(P^{(i)})  \right),
\end{equation}
while the \emph{dual learning problem} is given by
\begin{equation}\label{eq:dualLearning}
       \theta_d^* := \displaystyle \arg\min_{\theta_d}  \displaystyle \sum_{i=1}^M \mathcal{L} \left( \tilde \lambda_{\theta_d}(P^{(i)}) - \lambda^*(P^{(i)})  \right).
\end{equation}
\end{subequations}
We refer to $\tilde U_{\theta_p^*}(\cdot)$ and  $\tilde \lambda_{\theta_d^*}(\cdot)$ as the \emph{approximated primal policy} and \emph{approximated dual policy}, respectively. \blue{Depending on the choice of the learning algorithm, problems \eqref{eq:primalLearning} and \eqref{eq:dualLearning} can be computationally demanding to solve. Hence, those optimization problems are generally carried out offline.}

In general \blue{once \eqref{eq:primalLearning} and \eqref{eq:dualLearning} have been solved offline}, it is difficult to validate  feasibility and optimality of the approximated policies $\tilde U_{\theta_p^*}(\cdot)$ and $\tilde \lambda_{\theta_p^*}(\cdot)$. In the following, we describe a sampling based probabilistic verification scheme (\blue{offline, before deployment}, Section~\ref{sec:aPrioriGuarantees}), and a ``hard" deterministic verification scheme (\blue{online, during run-time}, Section~\ref{sec:aPosterioriGuarantees}).

\subsection{Probabilistic Safety and Performance Guarantees}\label{sec:aPrioriGuarantees}
In this section, we provide a methodology to verify the feasibility and optimality of the approximated policies \blue{offline after \eqref{eq:primalLearning} and \eqref{eq:dualLearning} are solved}. Specifically, given a desired maximum suboptimality level, we would like to verify that the approximated \blue{primal} policy satisfies this suboptimality level with high probability.  Formally, we define the  primal and dual objective functions
\begin{subequations}
\begin{align}
     p(P;U) &:= \textstyle \frac{1}{2}U^\top Q(P)U + c(P)^\top U \label{eq:primalObj}\\
     d(P;\lambda) &:= \frac{1}{2} \lambda^\top \tilde Q(P)\lambda + \tilde{c}(P)^\top\lambda + \tilde g(P) \label{eq:dualObj}.
\end{align}
\end{subequations}
Our goal is to ensure that \blue{the approximated policies are feasible and near-optimal with high probability, i.e.,}
\begin{subequations}
\begin{align}
    \mathbb{P}\big[ \hspace{.2cm} & H(P) \tilde U_{\theta_p^*}(P)\leq h(P), \label{eq:primGoal} \\
    &   p(P; \tilde U_{\theta_p^*}(P)) \leq J^*(P)+\gamma_p  \hspace{.2cm}   \big] \geq 1-\epsilon_p, \nonumber \\
    \mathbb{P}\big[   \hspace{.2cm}  &  \tilde \lambda_{\theta_d^*}(P)\geq0,  \label{eq:dualGoal} \\
    &   d(P; \tilde \lambda_{\theta_d^*}(P)) \geq J^*(P)-\gamma_d  \hspace{.2cm}  \big] \geq 1-\epsilon_d, \nonumber
\end{align}
\end{subequations}
where $\gamma_p$ ($\gamma_d$) are user-defined desired primal (dual) suboptimality levels, $\epsilon_p$ ($\epsilon_d$) are the user-defined maximum admissible primal (dual) violation probabilities, and $J^*(P)=p(P;U^*(P))=d(P;\lambda^*(P))=D^*(P)$ denotes the optimal value.

In general, evaluating \eqref{eq:primGoal}--\eqref{eq:dualGoal} is difficult since a multidimensional probability integral needs to be evaluated. In the following, we propose a sampling based method for verifying the above conditions. To this end, we extract $N_p$ primal and $N_d$ dual (verification) samples $\{P^{(i)}, J^*(P^{(i)})\}_{i}$. 
\begin{proposition}\label{prop:aPriori}
    Let $0<\beta_p\ll1$ and $0<\beta_d\ll1$ be desired primal and dual confidence levels, and $N_p \geq \frac{\ln(1/\beta_p)}{\ln(1/(1-\epsilon_p))}, N_d \geq \frac{\ln(1/\beta_d)}{\ln(1/(1-\epsilon_d))}$.
    If the following ``primal conditions" hold
    \begin{equation}\label{eq:primAPrioriVer}
    \left.
    \begin{array}{ll}
        & H(P^{(i)}) \tilde U_{\theta_p^*}(P^{(i)})\leq h(P^{(i)}),\\ 
        & p(P^{(i)}; \tilde U_{\theta_p^*}(P^{(i)})) \leq J^*(P^{(i)})+\gamma_p 
    \end{array}
    \right\} i=1,\ldots,N_p
    \end{equation}
    then \eqref{eq:primGoal} is satisfied with confidence at least $1-\beta_p$. Similarly, if the following ``dual conditions" hold
   \begin{equation}\label{eq:dualAPrioriVer}
    \left.
    \begin{array}{ll}
        & \tilde \lambda_{\theta_d^*}(P^{(i)})\geq0,\\ 
        & d(P^{(i)};\tilde \lambda_{\theta_d^*}(P^{(i)})) \geq J^*(P^{(i)})-\gamma_d 
    \end{array}
    \right\} i=1,\ldots,N_d,
    \end{equation}
    then \eqref{eq:dualGoal} is satisfied with confidence at least $1-\beta_d$.
\end{proposition}
\begin{proof}
    The proof is based on \cite[Theorem 3.1]{TempoAut1997LogOverLog}. Details are provided in the Appendix.
\end{proof}
Given $\epsilon_{p/d},\gamma_{p/d},\beta_{p/d}>0$, Proposition~\ref{prop:aPriori} lower bounds the sample sizes $N_p$ and $N_d$ that need to be used for verification of the trained policies. 
The so-called confidence levels $\beta_{p/d}$ are typically chosen very small ($10^{-6}\sim10^{-8}$), and only have a small impact on the sample sizes \blue{due to the $\log$-dependence}. Intuitively, they bound the probability of the verification schemes failing, see \cite{tempo:calafiore:dabbene} for details.
In practice, Proposition~\ref{prop:aPriori} can be used to decide if the policies require retraining. 

\blue{We point out that as both policy learning and sampling-based verification schemes are done offline prior to deployment of the approximated controllers, this allows us to train and verify, and potentially re-train, with large sample sizes to ensure accuracy of the learned policies $\tilde U_{\theta_p^*}(\cdot)$ and $\tilde \lambda_{\theta_p^*}(\cdot)$.}

\subsection{Online Feasibility and Optimality Certification}\label{sec:aPosterioriGuarantees}
Proposition~\ref{prop:aPriori} guarantees that the approximated primal and dual policies $\tilde{U}_{\theta_p}(\cdot)$ and $\tilde{\lambda}_{\theta_d}$ are feasible and near-optimal with probability at least $1-\epsilon_p$ and $1-\epsilon_d$. However, Proposition~\ref{prop:aPriori} is \emph{not} able to certify feasibility and optimality  of $\tilde U_{\theta_p^*}(\bar P)$ and $\tilde \lambda_{\theta_p^*}(\bar P)$ at a  \emph{given} parameter of $\bar{P}$. To address this issue, we next propose a ``hard" deterministic certification scheme, \blue{which is carried out online during run-time}. We begin with the following observation:
\begin{proposition}\label{prop:PDcheck}
    Given a parameter $P$, assume that $\tilde U_{\theta_p^*}(P)$ satisfies $H(P)\tilde U_{\theta_p^*}(P) \leq h(P)$ and $\tilde\lambda_{\theta_d^*}(P)\geq0$. Then, the suboptimality of $\tilde U_{\theta_p^*}(P)$ is bounded by
    \begin{equation}\label{eq:PDcheck}
            p(P;\tilde U_{\theta_p^*}(P)) - d(P;\tilde\lambda_{\theta_d^*}(P)),
    \end{equation}
    i.e., $p(P; \tilde U_{\theta_p^*}(P)) - J^*(P) \leq   p(P;\tilde U_{\theta_p^*}(x)) - d(P;\tilde\lambda_{\theta_d^*}(P))$.
\end{proposition}
\begin{proof}
    By assumption, $\tilde U_{\theta_p^*}(P)$ and $\tilde\lambda_{\theta_d^*}(P)$ are primal and dual feasible, respectively. By weak duality (Section~\ref{sec:dualitySection}), $d(P;\tilde\lambda_{\theta_d^*}(P)) \leq J^*(P)$, which concludes the proof.
\end{proof}
Proposition~\ref{prop:PDcheck} provides a computationally efficient way to estimate the suboptimality of the approximated primal policy $\tilde U_{\theta_p^*}(\cdot)$ without solving the optimization problem. 

We use Proposition~\ref{prop:PDcheck} in our framework as follows: Let $\gamma>0$ be the desired maximum suboptimality level. For a given parameter $P$, if \eqref{eq:PDcheck} is smaller than $\gamma$, then the primal solution $\tilde{U}_{\theta_p^*}(P)$ is guaranteed to be at most $\gamma$-suboptimal. In this case, the first element of  $\tilde U_{\theta_p^*}(P)$ is applied, and the procedure is repeated at the next time step. If \eqref{eq:PDcheck} is larger than the predetermined suboptimality level $\gamma$, then a backup controller is employed (see Section~\ref{sec:discussion} for a discussion).

Algorithm~\ref{alg:PDpolicyLearning} summarizes the proposed \emph{Primal-Dual Policy Learning} scheme. 
\begin{algorithm}[h!]
\caption{Primal-Dual Policy Learning}\label{alg:PDpolicyLearning}
\begin{algorithmic}[1]
\vspace{0.2em}

\Statex \hspace{-1.05em}
\textbf{Input:} max.~{violation} probability $\epsilon>0$, confidence level $0<\beta\ll 1$, suboptimality level $\gamma>0$

\Statex \hspace{-1.05em}
\textbf{Select:} $\epsilon_p,\epsilon_d>0$, $\epsilon_p+\epsilon_d=\epsilon$; $\beta_p,\beta_d>0$, $\beta_p+\beta_d=\beta$; $\gamma_p,\gamma_d>0$, $\gamma_p+\gamma_d=\gamma$
\Statex{} 

\Statex\hspace{-1.5em} \textit{Start learning process (offline)}
\Statex\hspace{-1em}\textbf{begin training}

\vspace{0.1em}
\hspace{-0.7cm}
\begin{minipage}[t]{\linewidth}
\begin{algorithmic}[1]
\State Learn primal policy $\tilde U_{\theta_p^*}(\cdot) \approx U^*(\cdot)$ as in \eqref{eq:primalLearning}
\State Learn dual policy $\tilde{\lambda}_{\theta_d^*} \approx {\lambda}^*(\cdot)$ as in \eqref{eq:dualLearning}
\State Validate $\tilde U_{\theta_p^*}(\cdot)$ and $\tilde \lambda_{\theta_d^*}(\cdot)$ using Proposition~\ref{prop:aPriori}
\State Repeat until \eqref{eq:primAPrioriVer} and \eqref{eq:dualAPrioriVer} satisfied
\end{algorithmic}
\end{minipage}
\vspace{0.01em}
\Statex \hspace{-1em}\textbf{end training} 

\Statex{} 

\Statex\hspace{-1.5em} \textit{Start control process (online)}
\Statex\hspace{-1em}\textbf{begin control loop} (for $t=0,1,\ldots$)

\vspace{0.1em}
\hspace{-0.7cm}
\begin{minipage}[t]{\linewidth}
\begin{algorithmic}[1]
\State Obtain $P_t$; evaluate $\tilde U_{\theta_p^*}(P_t)$ and $\tilde\lambda_{\theta_d^*}(P_t)$
\State If  $\tilde U_{\theta_p^*}(P_t)$ and $\tilde\lambda_{\theta_d^*}(P_t)$ feasible and \eqref{eq:PDcheck} $\leq\gamma$,  apply first element of $\tilde U_{\theta_p^*}(P_t)$
\State Else, apply back-up controller
\end{algorithmic}
\end{minipage}
\vspace{0.01em}
\Statex \hspace{-1em}\textbf{end control loop} {(until end of process)}
\end{algorithmic}
\end{algorithm}
The following theorem bounds the frequency in which the backup controller will be applied.
\begin{theorem}\label{thm:verification}
    Assume that $\tilde U_{\theta_p^*}(P_t)$ and $\tilde\lambda_{\theta_d^*}(P_t)$ satisfy the conditions ~\eqref{eq:primAPrioriVer}--\eqref{eq:dualAPrioriVer}, and let  $\beta$, $\beta_p$, $\beta_d$, $\epsilon$, $\epsilon_p$, $\epsilon_d$, $N_p$ and $N_d$ be chosen as in Algorithm~\ref{alg:PDpolicyLearning}. Then, with confidence at least $1-\beta$, it holds
    \begin{align*}
         \mathbb{P}\big[ \hspace{.2cm} & H(P) \tilde U_{\theta_p^*}(P)\leq h(P), \tilde \lambda_{\theta_d^*}(P)\geq0, \\
           & p(P;\tilde U_{\theta_p^*}(P)) - d(P;\tilde\lambda_{\theta_d^*}(P))  \leq \gamma \hspace{.2cm} \big ] \geq 1-\epsilon. 
    \end{align*}
\end{theorem}
\begin{proof}
   Application of the union bound \cite{bonferroni1936teoria} and Propositions~\ref{prop:aPriori}--\ref{prop:PDcheck} yields the desired result.
\end{proof}
{Theorem~\ref{thm:verification} ensures that, if the primal and dual policies have been validated according to Proposition~\ref{prop:aPriori}, then the primal policy $\tilde{U}_{\theta_p^*}(\cdot)$ will generate inputs that are at most $\gamma$-suboptimal with probability at least $(1-\epsilon)$, which can be checked during run-time using the ``\eqref{eq:PDcheck}$\leq \gamma$" condition.}

\subsection{Discussion}\label{sec:discussion}

\subsubsection*{Computational Complexity}
Since the primal-dual policy learning scheme only requires \blue{online} the \emph{evaluation} of the trained policies $\tilde{U}_{\theta_p^*}(\cdot)$ and $\tilde\lambda_{\theta_d^*}(\cdot)$, it can typically be executed significantly faster than solving the optimization problem \eqref{eq:MPCproblemOrig}. For example, if the policies are approximated through a combination of basis functions, i.e., $\tilde{U}_\theta(P) = \sum_i\theta_i \kappa_i(P)$, then the trained policies can be evaluated fully in parallel on embedded platforms such as FPGAs \cite{DomahidiLearning2011}. 


\subsubsection*{Backup Controller}
The choice of the backup controller, which is used in Algorithm~1 as a fall-back strategy,  is highly problem-dependent. Since backup controllers are generally available in practice, we do not see the use of a backup controller as a major limitation of our strategy. 




\section{Case Study: Integrated Chassis Control}\label{sec:simResults}
In this section we present a numerical case study in the field of integrated chassis control (ICC) for vehicles. Roughly speaking, the goal in chassis control is to improve \blue{a vehicle's dynamics and user comfort}  by actively controlling its motions. Traditionally, chassis control is carried out by many independent subsystems. More recently, however, there have been attempts to consider the coupling between each individual subsystem, giving rise to so-called integrated chassis control (ICC) \cite{ChandrasekanICC2011}. 
As such, MPC is a natural control strategy, since it allows the incorporation of input and state constraints in a disciplined manner, and is able to handle the multivariate nature of the task.

In the following, we study the problem of calculating the yaw moment, the roll moment and the lateral force \blue{in an integrated chassis control problem}. 

\subsection{Problem Formulation}
We consider a linear parameter-varying system of the form
\begin{align}\label{eq:icc_model_simple}
 x_{t+1} = A(v_t)x_t + B(v_t)u_t + E(v_t) \delta_t,\quad y_t = C x_t,  
\end{align}
where $x \in \R^4$ is the state, $u \in\R^3$ is the input, $y_t \in \mathbb{R}^3$ is the output, $v_t \in \mathbb{R}$ is the vehicle's longitudinal velocity, and $\delta_t \in \mathbb{R}$ is the driver's steering input, \blue{which we assume can be predicted}, see \cite{TAKANO2003149} for details on variable nomenclature.




The control objective is to minimize the output tracking error while satisfying input constraints and input rate constraints. Hence, the MPC problem is given by
\begin{equation}\label{eq:ICC_example}
    \begin{array}{llll}
        \displaystyle \min_{X_t,U_t} & \displaystyle \sum_{k=0}^{T-1} (y_{k|t}-y_{k|t}^\text{ref})^\top Q (y_{k|t}-y_{k|t}^\text{ref}) + u_{k|t}^\top R u_{k|t}  \\
        \ \ \text{s.t.}   & x_{k+1|t} = A(v_{t}) x_{k|t} + B(v_{t}) u_{k|t} + E(v_{t})\delta_{k|t}, \\
                          & y_{k|t} = C x_{k|t}, \ |u_{k|t}|\leq \bar{u},\ |u_{k|t} - u_{k-1|t}| \leq \xoverline{\Delta u},\\
                          & x_{0|t} = x_t,\ u_{-1|t} = u_{t-1},\ k=0,\ldots,T-1,
    \end{array}
\end{equation}
where $Q,R \in \mathbb{R}^{3 \times 3}$ are positive definite matrices, $y_{k|t}^\text{ref}\in\R^3$ are given reference signals, $\bar{u}$ defines the input constraints, $\xoverline{\Delta u}$ defines the rate constraints, and $u_{t-1}$ is the previous input. The parameters in \eqref{eq:ICC_example} are $P_t = (x_t, v_t, \{y_{k|t}^\text{ref},\delta_{k|t}\}_{k}, u_{-1|t})\in\R^{20}$. 
The velocity $v_t$ enters the dynamics in a nonlinear fashion, and so, the optimal control law $U^*(P_t)$ cannot be derived using standard methods from explicit MPC. In the following, we approximate $U^*(\cdot)$ using the approach described in Section~\ref{sec:primalDualPolicyLearning}. Throughout, we consider a horizon of $T=3$, which is typical in practice, such that $U^*: \R^{20} \to \R^9$.

\subsection{Offline Training and Verification} 
In this section, we illustrate our proposed primal-dual policy learning method on \eqref{eq:ICC_example} using a Deep Neural Network function approximator with Rectified Linear Unit (DNN-ReLU) activation \cite{Goodfellow-et-al-2016}. 
We train our primal and dual neural networks using $M = 1\,000$ samples\footnote{\blue{The training sample size is a tuning parameter whose choice depends on the specific problem instance.}}, uniformly sampled over a parameter set $\mathcal{P}$. We aim for a maximal admissible suboptimality level of $\gamma=1$, which should hold with probability at least $99\%$. Following Algorithm~\ref{alg:PDpolicyLearning}, we select $\gamma_p = \gamma_d = \gamma/2=0.5$ and $\epsilon_p = \epsilon_d = \epsilon/2=0.5\%$. The confidence levels are chosen to be $\beta_p=\beta_d=10^{-7}$, resulting in a verification sample size of $N_p=N_d=3216$ (Proposition~\ref{prop:aPriori}). Using trial and error, we found a DNN-ReLU of width 15 and depth $L=3$ for $\tilde U_{\theta_p^*}(\cdot)$ which satisfies the conditions in Proposition~\ref{prop:aPriori}, while a DNN-ReLU of width 5 and depth $L=3$ for $\tilde{\lambda}_{\theta_d^*}(\cdot)$ is found to satisfy the condition in Proposition~\ref{prop:aPriori}.


Once the policies are trained, to estimate their conservatism, we evaluate the suboptimality levels $\alpha_p := p(P; \tilde U_{\theta_p^*}(P)) - J^*(P)$, $\alpha_d := J^*(P) - d(P;\tilde \lambda_{\theta_d^*}(P))$ and $\alpha := p(P;\tilde U_{\theta_p^*}(P)) - d(P;\tilde\lambda_{\theta_d^*}(P))$ for 100'000 randomly extracted parameters $P^{(i)}$, and also determine the empirical violation probabilities $\hat \epsilon_p$, $\hat \epsilon_d$ and $\hat\epsilon$. 


\begin{table}[h!]
\caption{Empirical suboptimality levels and violation probabilities. }
\label{tab:subOpt}
\centering 
\ra{1.3}
\begin{tabular}{@{}l | l l l || c c c @{}}\toprule
 & $\alpha_p$ & $\alpha_d$ & $\alpha$ & $\hat\epsilon_p$ & $\hat\epsilon_d$ & $\hat\epsilon$ \\
 \midrule
 mean       & 0.0061     & 0.0129       & 0.0190        &   &    &  \\ 
 median     & 0.00083    & 0.0053       & 0.0084        & 0.045\%   & 0   & 0.005\% \\ 
 maximum    &  1.2552    & 0.2608       & 1.2732        &           &       &  \\
\bottomrule
\end{tabular}
\end{table} 

We see from Table~\ref{tab:subOpt} that the approximated primal and dual policies result in control laws that are close to optimal, with an median duality gap of $\alpha = 0.00083$, and worst-case duality gap of $\alpha = 1.2732$. Furthermore, we see that the empirical probability of the duality gap being larger than the predetermined $\gamma=1$ is $\hat{\epsilon}= 0.005\%$, which is much smaller than the targeted violation probability of $\epsilon=1\%$. This implies that, for this numerical example, the sample sizes provided by Proposition~\ref{prop:aPriori} are conservative; in practice, the trained policies perform better than expected. 


\subsection{Comparison with Online MPC}
In this section, we compare the trained DNN-ReLUs  with online MPC in terms of computation time. 
To solve online MPC, we use Gurobi and Mosek, two state-of-the-art solvers.
Table~\ref{tab:compTime} reports the computation time of our primal-dual policy approximation scheme using DNN-ReLU and the online MPC.
The timings are taken on an early 2016 MacBook, that runs on a 1.3 GHz Intel Core m7 processor and is equipped with 8 GB RAM and 512 GB SSD. We use MATLAB generated MEX files to determine the run-time of DNN-ReLU. 
The optimization problems are formulated in Yalmip \cite{lofberg:05}, and the timings are those reported by the solvers themselves. \blue{Notice that, comparing the results with explicit MPC is difficult, since explicit MPC generally only applies for linear time-invariant systems.} 

\begin{table}[h!]
\caption{Computation Times, rounded to two significant digits.}
\label{tab:compTime}
\centering 
\ra{1.3}
\begin{tabular}{@{}l | l  l l l @{}}\toprule
time [ms] & DNN-ReLU  & Gurobi & Mosek \\
 \midrule
 min.   &  0.020   & 1.3  & 2.1     \\ 
 max.   & 0.034    & 2.5  & 3.1     \\
 mean   &  0.023   & 1.5  & 2.4     \\
 std.   & 0.0023   & 0.25 & 0.23    \\
\bottomrule
\end{tabular}
\end{table} 



From Table~\ref{tab:compTime}, we see that DNN-ReLU is significantly faster than online MPC, \blue{with an average speed-up of over 65x compared to Gurobi, and over 100x compared to Mosek}. This is because evaluating a DNN-ReLU just involves simple matrix-vector multiplications and $\max$-operations, but no matrix inversions, as opposed to the case of numerical optimization. 
Moreover, we observe that the evaluation times obtained using DNN-ReLU are more consistent than those of Gurobi and Mosek, with a standard deviation of~10\% only. 

\section{Conclusion}\label{sec:conclusion}
In this paper, we proposed a new method for approximating the explicit MPC control law for linear parameter varying systems.
We propose to approximate the MPC controller directly using supervised learning techniques, and invoke two verification schemes to ensure safety and performance of the approximated controller. 
Since the proposed verification scheme only requires the evaluation of trained policies, our algorithm is computationally efficient, and can be implemented even on resource-constrained systems. Indeed, our numerical case study has revealed that the proposed primal-dual policy learning framework allows the Integrated Chassis Control problem to be solved up to 100x faster compared to state-of-the-art solvers, while maintaining performance.


\balance
\section*{ACKNOWLEDGMENT}
This research was partially funded by the Hyundai Center of Excellence
at the University of California, Berkeley. This work was also sponsored by the Office of Naval Research.
The authors thank Yi-Wen Liao and Dr. Jongsang Suh for helpful discussions on the integrated chassis control problem.

\begin{appendix}

\section*{Proof of Proposition~\ref{prop:aPriori}}
Let $H_j(P)$ and $h_j(P)$ denote the $j$th row of $H(P)$ and $h(P)$, respectively. For a given $\tilde U(\cdot)$, consider the following auxiliary function
\begin{align*}
    Q(P) := \max\Big\{ & \max_j\{ H_j(P)\tilde U(P) - h_j(P)   \}, \\[-1ex]
                        & p(P;\tilde U(P)) - J^*(P)-\gamma_p  \Big\},
\end{align*}
and define $\hat Q_N := \max_{i=1,\ldots,N}\{ Q(P^{(i)})\}$, where $\{P^{(i)}\}_{i}$ are a collection of independent samples drawn according to $\mathbb{P}$. It follows \cite[Theorem 3.1]{TempoAut1997LogOverLog} that, if $N \geq \frac{\ln1/\beta}{\ln1/(1-\epsilon)}$, then 
\begin{equation*}
    \mathbb{P}^N \big[ \mathbb{P}[Q(P) > \hat{Q}_N] \leq \epsilon \big] \geq 1-\beta
\end{equation*}
Proposition~\ref{prop:aPriori} now follows from the observation that $\hat Q_N=0$ is equivalent to \eqref{eq:primGoal}.

\end{appendix}


\renewcommand{\baselinestretch}{0.91}
\renewcommand{\baselinestretch}{1}

\section*{References}
{
\printbibliography[heading=none, resetnumbers=true]
}



\end{document}